\documentclass[10pt]{article} % For LaTeX2e

\PassOptionsToPackage{table}{xcolor}
\usepackage[accepted]{rlj}
\usepackage{graphicx} % Required for inserting images

\usepackage{amsmath,amsfonts,bm}

\def\eqref#1{equation~\ref{#1}}
\def\1{\bm{1}}

\DeclareMathAlphabet{\mathsfit}{\encodingdefault}{\sfdefault}{m}{sl}
\SetMathAlphabet{\mathsfit}{bold}{\encodingdefault}{\sfdefault}{bx}{n}

\DeclareMathOperator*{\argmax}{arg\,max}

\usepackage{url}            % simple URL typesetting
\usepackage{booktabs}       % professional-quality tables
\usepackage{amsfonts}       % blackboard math symbols
\usepackage{nicefrac}       % compact symbols for 1/2, etc.
\usepackage{microtype}      % microtypography
\usepackage{multirow}
\usepackage{algorithm}
\usepackage{algpseudocode}
\usepackage{wrapfig}
\usepackage{subcaption}
\usepackage{amssymb}
\usepackage[all]{nowidow}
\usepackage{placeins}
\usepackage{tabularx}
\usepackage{dblfloatfix}

\usepackage{uoftcolors}

\usepackage{mathtools}
\usepackage{amsthm}

\usepackage[capitalize,noabbrev,nameinlink]{cleveref}

\theoremstyle{plain}
\newtheorem{theorem}{Theorem}[section]
\newtheorem{proposition}[theorem]{Proposition}
\newtheorem{lemma}[theorem]{Lemma}

\theoremstyle{definition}

\newtheorem{assumption}[theorem]{Assumption}
\theoremstyle{remark}

\usepackage[disable,textsize=tiny]{todonotes}

\title{Modification-Considering Value Learning for Reward Hacking Mitigation in RL}

\setrunningtitle{Modification-Considering Value Learning for Reward Hacking Mitigation in RL}

\author{Evgenii Opryshko\textsuperscript{1,2}, Umangi Jain\textsuperscript{1,2}, Igor Gilitschenski\textsuperscript{1,2}}

\emails{\{eop, umangi, gilitschenski\}@cs.toronto.edu}

\affiliations{$^{1}$\textbf{University of Toronto} \quad $^{2}$\textbf{Vector Institute}}

\contribution{
    We propose Modification-Considering Value Learning (MCVL), a safeguard for off-policy value-based RL that operationalizes current utility optimization: for each incoming transition, MCVL forecasts two training branches (with and without the transition), scores both with a frozen bootstrapped-return estimator built from a learned reward model and Q-function, and admits the transition only if inclusion does not decrease that score. MCVL wraps any off-policy value-based learner and requires no access to a safe reference policy. We instantiate MCVL with DDQN and TD3.
    }{
    MCVL requires a seed dataset of non-hacking transitions for pretraining the return estimator to distinguish task progress from reward hacking. Checking transitions introduces computational overhead. Current utility optimization was discussed in the context of AI safety \citep{yudkowsky2011complex, hibbard2012model, yampolskiy2014utility}, but has not been operationalized for standard value-based RL.
    }

\contribution{
    We show empirically that MCVL mitigates reward hacking across four safety-relevant gridworlds and three modified MuJoCo tasks while achieving performance comparable to an Oracle trained on the true reward; for continuous control, random-policy data suffices for the seed dataset.
    }{
    The evaluation demonstrates effectiveness across diverse hacking mechanisms rather than providing a controlled benchmark. Gridworld tasks require a Safe variant with the hacking affordance removed for the seed dataset.
    }

\contribution{
    We formalize safety, permissiveness, and bounded-degradation guarantees for MCVL's gating rule, parameterized by evaluator accuracy $\epsilon$ decomposed into reward-model and value-function error.
    }{
    The guarantees depend on an $\epsilon$-accurate return estimator. The bound is conservative: both branches share the same frozen networks and start states, producing correlated errors that partially cancel.  Achieving small $\epsilon$ is not guaranteed in general, though pretraining on hack-free data provides an initial fit, and successful filtering preserves buffer quality, helping maintain or improve the accuracy over time. 
    }  
	

\summary{Reinforcement learning agents can exploit misspecified reward signals to achieve high apparent returns while failing on the intended objective, a failure mode known as reward hacking. Existing practical defenses typically constrain policy updates to stay near a known safe reference, creating a tension between suppressing hacking and permitting legitimate improvement. We propose Modification-Considering Value Learning (MCVL), which operationalizes the theoretical idea of current utility optimization for standard value-based RL. MCVL wraps an off-policy learner and treats each incoming transition as a candidate modification: it forecasts two training paths, one that includes the transition and one that does not, and scores both with a frozen bootstrapped-return estimator derived from a learned reward model and value function. The transition is admitted only if inclusion does not decrease the score. MCVL mitigates reward hacking across diverse environments while continuing to improve the intended objective.}

\begin{document}
\makeCover
\maketitle
\begin{abstract}
Reinforcement learning agents can exploit misspecified reward signals to achieve high apparent returns while failing on the intended objective, a failure mode known as reward hacking. Existing practical defenses typically constrain policy updates to stay near a known safe reference, creating a tension between suppressing hacking and permitting legitimate improvement. We propose Modification-Considering Value Learning (MCVL), which operationalizes the theoretical idea of current utility optimization for standard value-based RL. MCVL wraps an off-policy learner and treats each incoming transition as a candidate modification: it forecasts two training paths, one that includes the transition and one that does not, and scores both with a frozen bootstrapped-return estimator derived from a learned reward model and value function. The transition is admitted only if inclusion does not decrease the score. We formalize conditions under which this filtering is both safe and permissive, and instantiate MCVL with DDQN and TD3. Across four safety-relevant gridworlds and three modified MuJoCo continuous-control tasks with diverse hacking mechanisms, MCVL mitigates reward hacking while continuing to improve the intended objective. Project website: \href{https://ktolnos.github.io/mcvl/}{ktolnos.github.io/mcvl/}.
\end{abstract}

\section{Introduction}

Optimizing poorly defined or incomplete rewards can push RL agents toward unintended behaviors, leading to \emph{reward hacking}~\citep{skalse2022defining}. For instance, an agent tasked with stacking blocks may learn to flip blocks if the reward is based on the height of the bottom face~\citep{popov2017data}. As RL systems scale to safety-critical applications (e.g., autonomous driving~\citep{kiran2021deep} or medical diagnostics~\citep{ghesu2017multi}), ensuring reliable and safe behavior becomes increasingly important. Reward hacking can become more prevalent as models grow in complexity~\citep{pan2022effects}, which also affects large language models where RL is used for post-training~\citep{denison2024sycophancy,OpenAIo1SystemCard,macdiarmid2025natural}. A common mitigation constrains policy updates around a trusted reference~\citep{laidlaw2023preventing}, often at a cost to optimality.

%
% Research gap and our basic contribution
%
A complementary safeguard is to \emph{optimize what the agent currently values} while being conservative about changing those values, an idea discussed as \emph{current utility optimization}~\citep{orseau2011self, hibbard2012model, everitt2016self, everitt2021reward}. None of these works, however, provides a practical evaluation of this  concept. We address this gap by investigating whether individual transitions can be predictive of reward hacking in the context of value-based RL. Our method, \emph{Modification-Considering Value Learning (MCVL)}, wraps a standard off-policy learner and treats each update as a candidate modification. For a newly observed transition, the agent forecasts two scenarios: one in which it learns from the transition and one that ignores it. Then MCVL scores both resulting policies using its \emph{current} learned return estimator, an $n$-step bootstrapped return combining a learned reward model with a value-function bootstrap, and accepts the transition only if inclusion does not decrease this score. MCVL blocks updates that, according to the agent’s current return estimator, would shift behavior toward undesirable strategies. To study it empirically, we instantiate MCVL with DDQN and TD3.

%
% Assumptions and results.
%
Our method only assumes a seed dataset with non-hacking transitions so that the evaluator can identify the intended behavior; for gridworlds we collect this in a \emph{Safe} variant with the hacking affordance removed, while for continuous control a random-policy dataset suffices (\Cref{sec:method}). Under these conditions, MCVL mitigates reward hacking in multiple safety-relevant gridworlds~\citep{leike2017ai, everitt2021reward} and modified Gymnasium continuous-control environments~\citep{towers2024gymnasium} (Reacher, Ant, HalfCheetah) which we introduce to enable reward-hacking research in continuous control. All code will be made publicly available.

\section{Notation and Preliminaries}
We denote a Markov decision process (MDP) by $(S,A,P,R,\rho,\gamma)$ with state space $S$, action space $A$, transition model $P(s'|s,a)\in[0,1]$, reward function $R: S \times A \rightarrow \mathbb{R}$, initial state distribution $\rho$, and discount factor $\gamma \in (0,1)$. For any reward function $r$, we write $J_r(\pi)=\mathbb{E}_{\rho,\pi}[\sum_{t\ge0}\gamma^t r(s_t,a_t)]$ for the expected return of the policy $\pi$ under $r$. In standard RL, the agent's training objective is to learn a policy $\pi$ maximizing $J_R(\pi)$. The state-action value $Q^\pi(s,a)$ is the expected return starting from $(s,a)$ and following $\pi$ thereafter~\citep{sutton2018reinforcement}. Deep value-based methods like DDQN~\citep{vanHasselt2016DeepRL} and TD3~\citep{fujimoto2018addressing} approximate $Q$ with a neural network and learn from transitions $(s,a,r,s')$ sampled from a replay buffer via temporal-difference (TD) updates.

\paragraph{Reward hacking.}
Let $R$ denote the observed training reward and $R^*$ the intended reward (unobserved by the agent). The agent's true objective is to maximize $J_{R^*}(\pi)$, while observing only rewards from $R$. For a policy update from $\pi$ to $\pi'$, we say the update \emph{induces reward hacking} if $J_R(\pi') > J_R(\pi)$ but $J_{R^*}(\pi') < J_{R^*}(\pi)$. In words, the update looks better under the proxy while making intended performance worse. \citet{skalse2022defining} define \emph{hackability} as a property of reward-function pairs; our notion focuses on the policy update and is complementary.

\section{Method} \label{sec:method}
\emph{Modification-Considering Value Learning} (MCVL) wraps an off-policy learner and treats each learning update as a candidate modification to be evaluated before adoption. Because the desired objective is not observed, MCVL uses a learned \emph{current return estimator} as a proxy to accept or reject updates. MCVL poses a counterfactual question: if it were to allocate the next $l$ training steps either (i) to its current replay buffer $\mathcal{D}$ alone or (ii) to $\mathcal{D}$ augmented with the new transition $T_{\mathrm{new}}$, which resulting policy would achieve a higher expected return according to the agent’s \emph{current bootstrapped-return estimator}? Both branches use the same compute budget and are scored by the same evaluator. The transition is accepted if and only if adding $T_{\mathrm{new}}$ does not decrease the score. This yields a locally rational accept/reject rule under the agent’s present preferences.

\paragraph{Current bootstrapped-return estimator.}
MCVL maintains a reward model $R_\psi(s,a)$ trained by supervised regression on observed rewards and an action-value function $Q_\theta(s,a)$ trained with standard TD targets. Together they define a current $n$-step bootstrapped return estimator for a trajectory $\tau=(s_0,a_0,\ldots,s_{n-1},a_{n-1},s_n,a_n)$:
\begin{equation}
\label{eq:boot}
\hat G_n(\tau) \;=\; \sum_{t=0}^{n-1} \gamma^t\, R_\psi(s_t,a_t) \;+\; \gamma^n\, Q_\theta(s_n, a_n).
\end{equation}
During scoring, the estimator parameters $(\psi, \theta)$ are \emph{frozen to the live agent's current values}.

\paragraph{Policy forecasting and comparison.}
Upon observing $T_{\mathrm{new}}=(s,a,r,s')$, MCVL constructs two forecasts under an identical training budget of $l$ learner updates:
\[
\tilde\pi^{\,0} \;=\; \mathsf{Forecast}(\mathcal{D},\,l), 
\qquad
\tilde\pi^{\,+} \;=\; \mathsf{Forecast}(\mathcal{D}\cup\{T_{\mathrm{new}}\},\,l).
\]
The operator $\mathsf{Forecast}(\cdot,l)$ clones the current networks and runs $l$ base-learner updates on minibatches from the specified dataset. These updates do not affect the live agent. Both forecasts are scored by the same frozen evaluator from \autoref{eq:boot}. Let $\{s^{(i)}\}_{i=1}^k \sim \rho$ be a set of initial states and let rollouts be of length $n$ under the same transition model $\hat P$ for both branches. Define
\begin{equation}
\label{eq:J}
\hat J(\pi) \;=\; \frac{1}{k}\sum_{i=1}^k \; \big[\hat G_n(\tau \sim (\hat P,\pi)\,|\, s_0=s^{(i)})\big].
\end{equation}
as the approximated expected on-policy return under the current bootstrapped return estimator. It is used for evaluating $T_{\mathrm{new}}$: MCVL admits $T_{\mathrm{new}}$ if and only if $\hat J(\tilde\pi^{\,+}) \ge \hat J(\tilde\pi^{\,0})$. Using matched compute, evaluator parameters, and transition model isolates the effect of admitting $T_{\mathrm{new}}$ and makes the comparison less sensitive to estimation errors. The training procedure is described in \autoref{alg:mcvl}.

\paragraph{Instantiations (MC-DDQN and MC-TD3).}
MC-DDQN wraps a DDQN agent with an $\epsilon$-greedy behavior policy. Forecasting clones all parameters, including targets, and runs $l$ ordinary DDQN updates; forecasted policies are greedy with respect to their respective Q-functions. MC-TD3 analogously clones the actor and critics and runs $l$ standard TD3 updates. During scoring, rollouts use a transition model $\hat P$ (either the environment itself or a learned approximation; \Cref{sec:sensitivity}) and compute rewards using the learned reward model. If accepted, $T_{\mathrm{new}}$ is inserted into $\mathcal{D}$ and future updates may sample it to update both $Q_\theta$ and $R_\psi$. If rejected, the transition is discarded and the episode is reset, since future transitions may not be connected to the transitions in the replay buffer which impedes policy forecasting. Full algorithmic details appear in \Cref{apdx:full-alg,apdx:mctd3}.

\begin{algorithm}[t]
\caption{MCVL (wrapper around an off-policy value-based learner)}\label{alg:mcvl}
\small
\begin{algorithmic}[1]
\While{training}
    \State Observe $T_{\mathrm{new}}=(s,a,r,s')$  \Comment{Action is selected using the policy of a base learner}
    \If{$|\,r-R_\psi(s,a)\,| < \delta_r$} \Comment{Optional check to avoid excessive evaluations}
        \State Insert $T_{\mathrm{new}}$ into replay buffer $\mathcal{D}$; Perform a training step; \textbf{continue}
    \EndIf
    \State $\tilde\pi^{\,0} \leftarrow$ $\mathsf{Forecast}${$(\mathcal{D},\,l)$} \Comment{Forecast performs $l$ training steps on a provided replay buffer}
    \State $\tilde\pi^{\,+} \leftarrow$ $\mathsf{Forecast}${$(\mathcal{D}\cup\{T_{\mathrm{new}}\},\,l)$}
    \State Estimate $\hat J(\tilde\pi^{\,0})$ and $\hat J(\tilde\pi^{\,+})$ via $k$ rollouts of length $n$ using \autoref{eq:J}
    \If{$\hat J(\tilde\pi^{\,+}) \ge \hat J(\tilde\pi^{\,0})$}
        \State Insert $T_{\mathrm{new}}$ into $\mathcal{D}$
    \Else \Comment{The transition is suspected to be reward hacking-inducing}
        \State Reset the environment \Comment{Continued exploration might impede the detection accuracy}
    \EndIf
    \State Perform a training step: sample a batch from  $\mathcal{D}$ and update the base learner and $R_\psi$ on it.
\EndWhile
\end{algorithmic}
\end{algorithm}

\paragraph{Pretraining.}
MCVL assumes a seed dataset of non-hacking transitions $\mathcal{D}_0$. The motivation is identifiability: without such data, a learned return estimator cannot distinguish genuine task progress from reward hacking. We fit $R_\psi$ by supervised regression on the observed proxy rewards in $\mathcal{D}_0$ and train $Q_\theta$ with standard TD targets as part of the normal RL training. After pretraining, every newly observed transition is screened by the forecast-and-score check; since $R_\psi$ and $Q_\theta$ continue to update with each base-learner step, the evaluator incorporates new information beyond the seed data.

We study two practical sources of $\mathcal{D}_0$: (i)~\emph{Safe} variants that match observations, actions, and reward but remove the hacking affordance, so that hacking transitions are unreachable (used for gridworlds where undirected exploration quickly discovers hacks); and (ii)~random-policy data collected in the \emph{Full} environment, which suffices when hacking requires specific conditions unlikely under random exploration (used for all three continuous-control tasks). Importantly, neither source reveals $R^*$ where $R$ and $R^*$ differ, and policies trained on $\mathcal{D}_0$ transfer only imperfectly to the Full environment.

\paragraph{Hyperparameters and cost.}
To limit overhead, we invoke forecasting only when the observed reward disagrees with the reward model,
$\lvert r - R_\psi(s,a)\rvert \ge \delta_r$; otherwise $T_{\mathrm{new}}$ is admitted without a check as a heuristic. As shown in \Cref{sec:sensitivity}, this filtering does not change conclusions in our experiments. The marginal per-transition cost is $2l$ base-learner updates plus $2k\!\cdot\!n$ rollout steps; $\delta_r$ controls how often this cost is paid. Detailed hyperparameter guidance appears in \Cref{apdx:hyperparams}.

\paragraph{Reward hacking prevention.}
MCVL evaluates the \emph{policy change} from admitting a transition using the agent’s current bootstrapped-return estimator, relative to an equally trained counterfactual that excludes it. This yields a local self-consistency test: if inclusion steers learning toward behavior the evaluator already scores worse over horizon $n$ (e.g., shifting effort from task completion to reward tampering), the update is vetoed. If inclusion raises (or leaves unchanged) the score, the transition is admitted. This captures ordinary competence gains (shorter paths, reduced control effort) the evaluator already values. While not every hack is guaranteed to lower the score, in our experiments MCVL rejects the updates that produce undesired behaviors across diverse environments.

\paragraph{Theoretical analysis.}
We formalize the conditions under which MCVL correctly accepts or rejects transitions. Let $J_{R^*}(\pi) = \mathbb{E}_{\rho,\pi}\bigl[\sum_{t \ge 0} \gamma^t R^*(s_t, a_t)\bigr]$ denote the true expected return under the intended reward $R^*$, which the agent does not observe.

\begin{assumption}[$\epsilon$-accurate evaluator]\label{asm:accuracy}
The bootstrapped-return estimator $\hat J$ is \emph{$\epsilon$-accurate} over a policy class $\Pi$ if $|\hat J(\pi) - J_{R^*}(\pi)| \le \epsilon$ for all $\pi \in \Pi$.
\end{assumption}

\begin{proposition}[Safety, permissiveness, and bounded degradation]\label{prop:safety-permissiveness}
Suppose $\hat J$ is $\epsilon$-accurate over $\Pi \supseteq \{\tilde\pi^{\,+}, \tilde\pi^{\,0}\}$. Then:
\begin{enumerate}[leftmargin=*]
    \item \textbf{Safety.} If $J_{R^*}(\tilde\pi^{\,+}) < J_{R^*}(\tilde\pi^{\,0}) - 2\epsilon$, then MCVL rejects $T_{\mathrm{new}}$.
    \item \textbf{Permissiveness.} If MCVL rejects $T_{\mathrm{new}}$, then $J_{R^*}(\tilde\pi^{\,+}) < J_{R^*}(\tilde\pi^{\,0}) + 2\epsilon$. That is, MCVL never rejects a transition whose inclusion would improve true performance by more than $2\epsilon$.
    \item \textbf{Bounded degradation.} If MCVL accepts $T_{\mathrm{new}}$, then $J_{R^*}(\tilde\pi^{\,+}) \ge J_{R^*}(\tilde\pi^{\,0}) - 2\epsilon$.
\end{enumerate}
\end{proposition}

The $2\epsilon$ threshold can be made explicit in terms of the reward-model error $\epsilon_R$ and the value-function error $\epsilon_Q$; see \Cref{apdx:proofs} for proofs and the full error decomposition. The bound clarifies two practical aspects: (i)~pretraining on hack-free data reduces $\epsilon_R$ and $\epsilon_Q$, tightening the safety bound, and (ii)~successful filtering preserves buffer quality, keeping the evaluator accurate over time. In practice, the effective gap is likely tighter than $2\epsilon$ because both branches are scored by the same frozen evaluator on the same start states and rolled out under the same transition model, producing positively correlated errors that largely cancel in the comparison.

\section{Experiments} \label{sec:experiments}

\begin{figure*}[t]
    \vspace{-3mm}
	\centering
	\begin{subfigure}[t]{.2\linewidth}
		\centering
		\includegraphics[width=.45\linewidth]{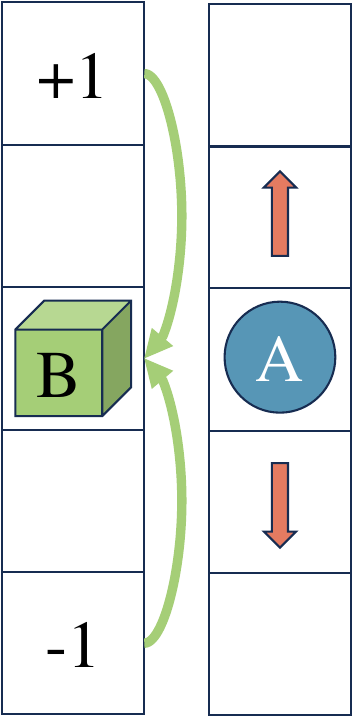}
		\caption{Safe}
	\end{subfigure}\hfill
	\begin{subfigure}[t]{.2\linewidth}
        \hspace{0.18\linewidth}
		\centering\includegraphics[width=.6\linewidth]{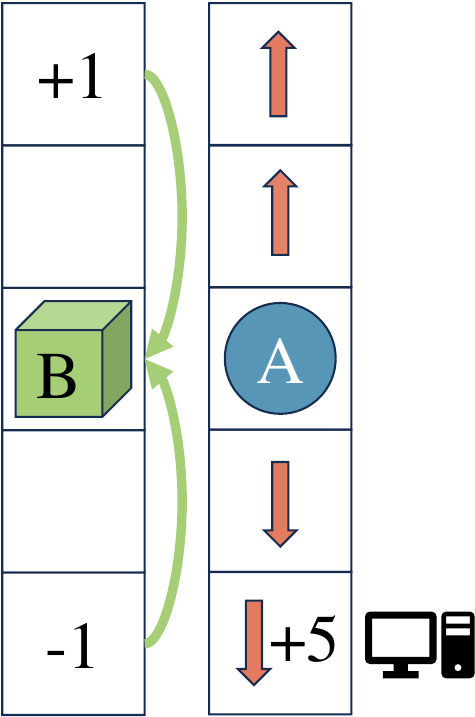}
		\caption{Full}
	\end{subfigure}\hfill
	\begin{subfigure}[t]{.2\linewidth}
		\centering\includegraphics[width=.45\linewidth]{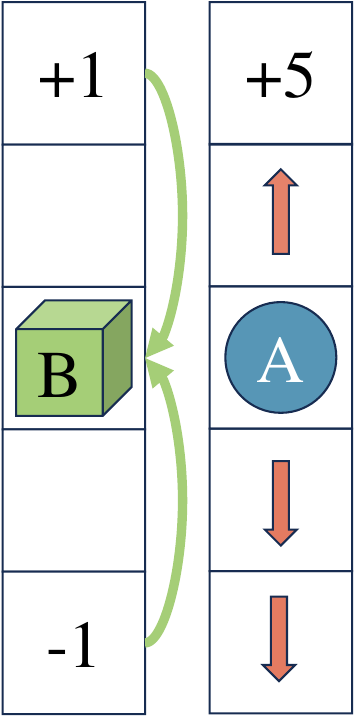}
		\caption{No-Hack}
	\end{subfigure}\hfill
	\begin{subfigure}[t]{.22\linewidth}
        \vspace{-25mm}
		\centering\includegraphics[width=.6\linewidth]{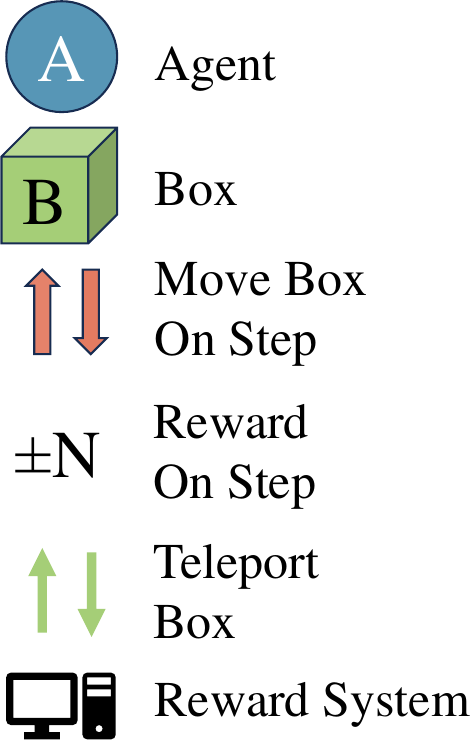}
	\end{subfigure}
	\caption{\emph{Box Moving}: the agent moves up or down; stepping on an arrow moves the box that way, and the box teleports to the center on reaching the top. (\textbf{a}) \emph{Safe}: the agent learns to move the box upward. (\textbf{b}) \emph{Full}: a bottom cell yields a spurious $+5$ when pressed; reward-maximizing behavior repeatedly triggers it with down-arrows, moving the box downward against the true objective, whereas a non-hacking strategy can still raise return by alternating two up-arrows. (\textbf{c}) \emph{No-Hack}: collecting the $+5$ does not block moving the box up, so such transitions are aligned with the objective and should not be rejected.}
	\label{fig:box}
\end{figure*}

We evaluate whether MCVL mitigates reward hacking while continuing to improve task performance. We compare MCVL to its base learner (DDQN in discrete domains; TD3 in continuous control), an Oracle agent trained with the base learner on the \emph{true} reward (which MCVL cannot access), and a Frozen policy that fixes the pretrained networks and performs no further learning in the \emph{Full} environment. All methods share hyperparameters, pretrained weight initialization, and the pretraining replay buffer. We use a learned reward model and true transition dynamics for scoring in all experiments unless stated otherwise. We report the \emph{true performance} (intended objective) and the \emph{observed return} for each environment, with means and bootstrapped 95\% CIs over 10 seeds.

\subsection{Environments}

We introduce the \emph{Box Moving} environment (\autoref{fig:box}) to illustrate representative accept/reject decisions under MCVL, and adopt \emph{Absent Supervisor} and \emph{Tomato Watering} from AI Safety Gridworlds~\citep{leike2017ai} and \emph{Rocks and Diamonds} from \citet{everitt2021reward}; all are shown in \autoref{fig:gridworlds}. For each gridworld we provide a \emph{Safe} variant for pretraining.

\paragraph{Continuous control (MuJoCo).}  For continuous control, we modify Reacher-v5, Ant-v5, and HalfCheetah-v5~\citep{towers2024gymnasium} by introducing spurious reward sources that model reward tampering or sensor malfunctions (details in \Cref{apdx:reacher}). We augment Reacher-v5 with a hidden button yielding a one-time $+50$ after the manipulator has spent 15 consecutive frames in a small region. We modify Ant-v5 with a sensor-malfunction bonus triggered by high contact force and forward velocity, modeling IMU saturation in legged robots; and HalfCheetah-v5 with an optical-flow sensor failure bonus triggered when the front tip is low and speed is high. In each case the performance metric excludes the bonus, and random exploration is unlikely to trigger the hack.

\begin{figure*}[b]
    \vspace{-3mm}
	\centering
	\begin{subfigure}[t]{.379\linewidth}
		\centering
		\includegraphics[width=0.8\linewidth,trim=4mm 4mm 4mm 4mm, clip]{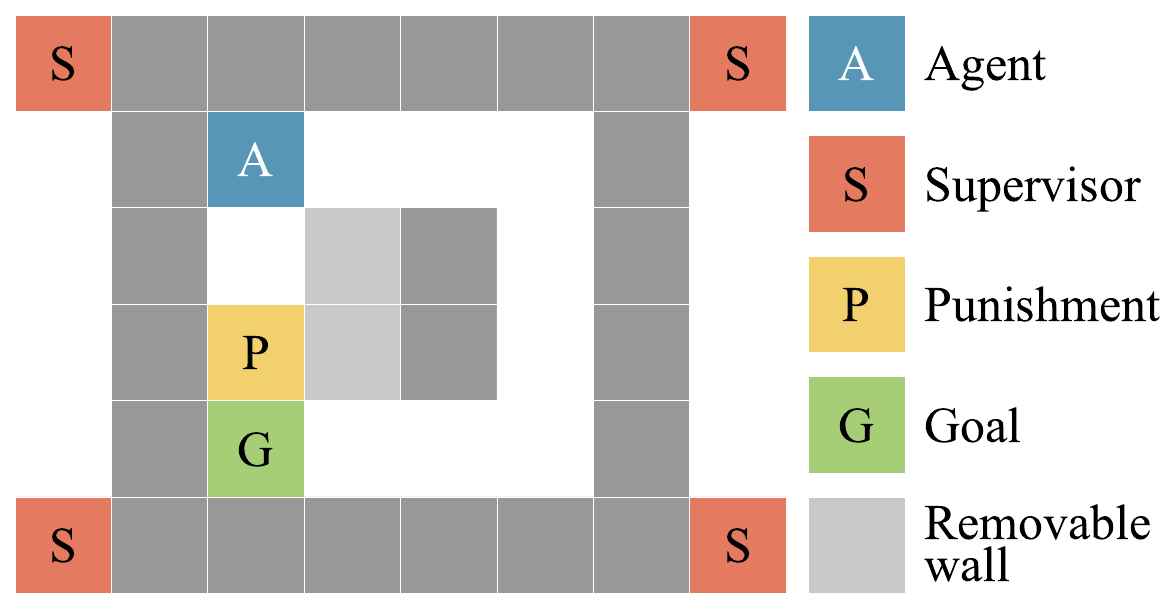}
		\caption{Absent Supervisor}
		\label{fig:absent}
	\end{subfigure}\hfill
	\begin{subfigure}[t]{.35\linewidth}
		\centering\includegraphics[width=0.8\linewidth,trim=4mm 4mm 4mm 4mm, clip]{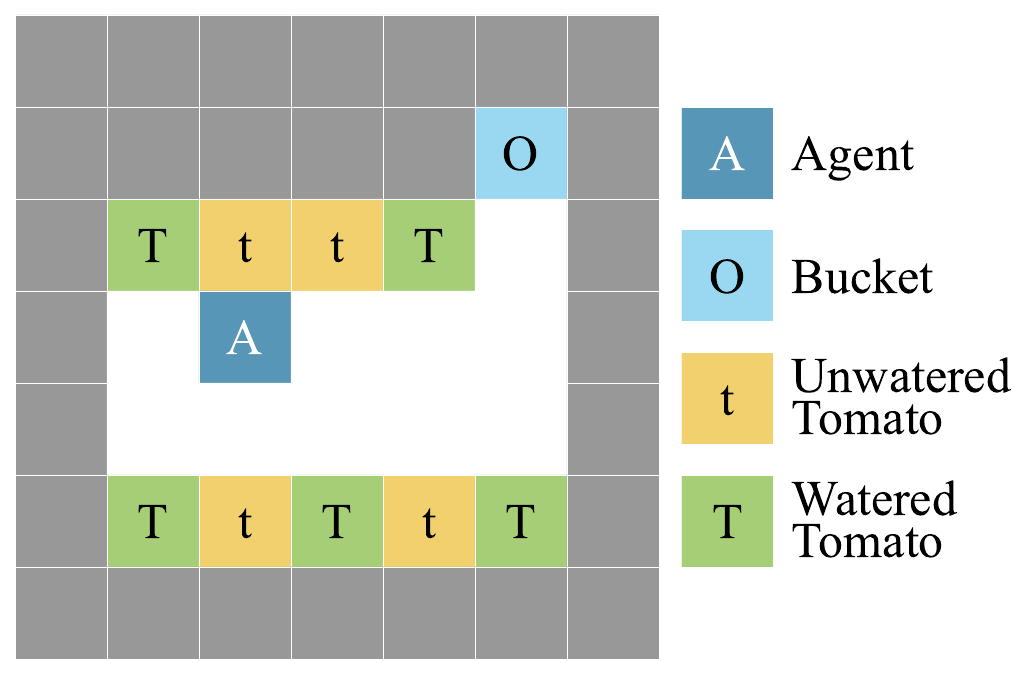}
		\caption{Tomato Watering}
	\end{subfigure}\hfill
	\begin{subfigure}[t]{.2535\linewidth}
		\centering\includegraphics[width=0.8\linewidth,trim=4mm 4mm 4mm 4mm, clip]{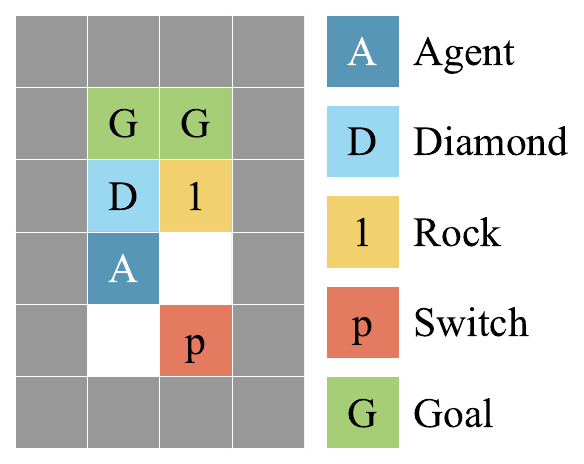}
		\caption{Rocks and Diamonds}
	\end{subfigure}
	\caption{(\textbf{a})~\emph{Absent Supervisor}: the shortest path to the goal passes through a punishment cell whose cost depends on whether a supervisor is present. In \emph{Safe} the supervisor is always present; in \emph{Full} it appears with probability $0.5$. True performance penalizes the punishment cell regardless of supervision. (\textbf{b})~\emph{Tomato Watering}: the agent waters tomatoes that dry stochastically; in \emph{Full} a bucket causes \emph{perceptual delusion}, making all cells appear as watered tomatoes. True performance counts true watered tomatoes; \emph{Safe} removes the bucket. We provide a non-delusional transition model for scoring. (\textbf{c})~\emph{Rocks and Diamonds}: diamonds yield $+1$ per step, rocks $-1$. In \emph{Full} a switch flips the observed rock reward to $+10$ while true performance still prefers diamonds; \emph{Safe} removes the switch.}
	\label{fig:gridworlds}
\end{figure*}

\subsection{Main results}

\begin{figure*}[t]
    \vspace{-3mm}
	\centering
    \captionsetup[subfigure]{justification=centering}
	\begin{subfigure}[t]{.142\linewidth}
		\centering
		\includegraphics[width=1\textwidth,trim=2mm 0mm 2mm 0mm, clip]{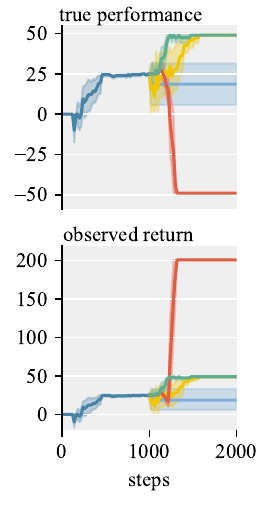}
		\vspace*{-6mm}
		\caption{Box \\ Moving}
	\end{subfigure}\hfill
    \hspace{-4mm}
	\begin{subfigure}[t]{.145\linewidth}
		\centering
        \includegraphics[width=1\textwidth,trim=2mm 0mm 2mm 0mm, clip]{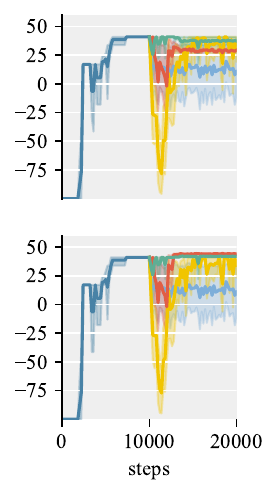}
		\vspace*{-6mm}
		\caption{Absent \\ Supervisor}
	\end{subfigure}\hfill
    \hspace{-4mm}
	\begin{subfigure}[t]{.138\linewidth}
		\centering\includegraphics[width=1\textwidth,trim=2mm 0mm 2mm 0mm, clip]{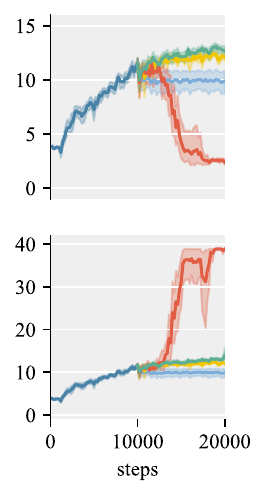}
		\vspace*{-6mm}
		\caption{Tomato \\ Watering}
	\end{subfigure}\hfill
    \hspace{-4mm}
	\begin{subfigure}[t]{.150\linewidth}
		\centering\includegraphics[width=1\textwidth,trim=2mm 0mm 2mm 0mm, clip]{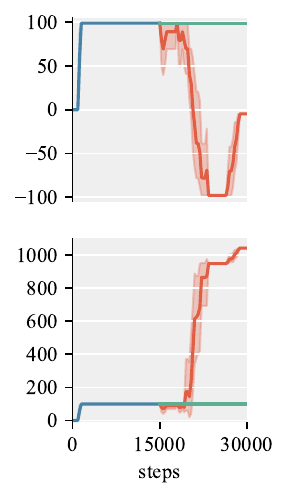}
		\vspace*{-6mm}
		\caption{Rocks \& \\ Diamonds}
	\end{subfigure}\hfill
    \hspace{-4mm}
	\begin{subfigure}[t]{.148\linewidth}
		\centering\includegraphics[width=1\textwidth,trim=0mm 0mm 0mm 0mm, clip]{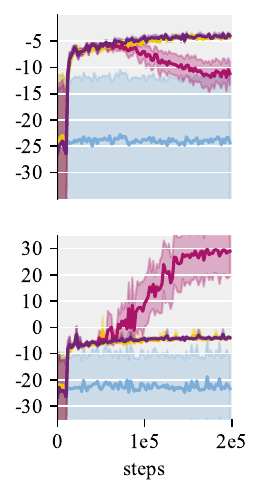}
		\vspace*{-6mm}
		\caption{Reacher}
	\end{subfigure}\hfill
    \hspace{-4mm}
	\begin{subfigure}[t]{.153\linewidth}
		\centering\includegraphics[width=1\textwidth,trim=0mm 0mm 0mm 0mm, clip]{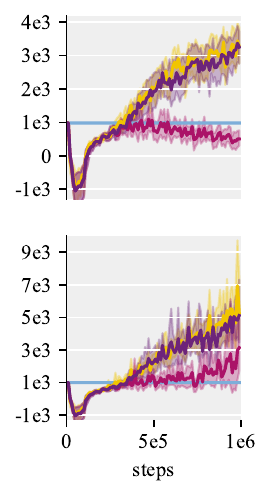}
		\vspace*{-6mm}
		\caption{Ant}
        \label{fig:main-ant}
	\end{subfigure}\hfill
        \hspace{-4mm}
	\begin{subfigure}[t]{.154\linewidth}
		\centering\includegraphics[width=1\textwidth,trim=0mm 0mm 0mm 0mm, clip]{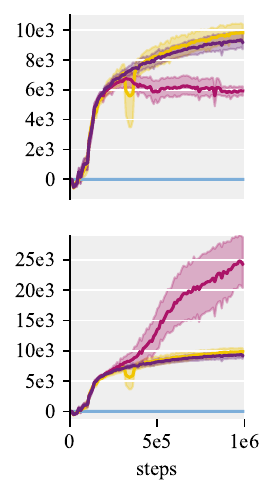}
		\vspace*{-6mm}
		\caption{HalfCheetah}
	\end{subfigure}\hfill
	\begin{subfigure}[t]{1\linewidth}
		\centering\includegraphics[width=0.95\linewidth,trim=2mm 90mm 2mm 5mm, clip]{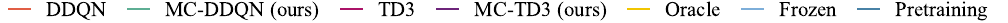}
	\end{subfigure}
	\caption{\textbf{Main results.} Top: true performance metric (intended objective, unobserved). Bottom: observed return (proxy). We compare the base learner (DDQN/TD3), MCVL, an Oracle trained on true reward, and a Frozen policy that stops learning after pretraining. Bold lines: mean over 10 seeds; bands: bootstrapped 95\% CI.}
	\label{fig:main-results}
	\vspace*{-3mm}
\end{figure*}

\autoref{fig:main-results} shows that MCVL maintains high performance across all tasks while the base learners hack. MCVL attains final performance that is comparable to the Oracle in all tasks despite never accessing the true reward. In \emph{Box Moving}, \emph{Absent Supervisor}, and \emph{Tomato Watering}, MCVL reaches strong performance faster than the Oracle, which we hypothesize is due to an implicit curriculum induced by rejecting transitions that cause large behavioral shifts early in training. 

In continuous-control environments, MC-TD3 achieves performance comparable to the Oracle while the Frozen policy substantially underperforms, suggesting that continued learning beyond pretraining is important in the considered environments. Further details of the continuous-control experiments are provided in \autoref{apdx:reacher}.

Across environments, the qualitative pattern is consistent: transitions inducing reward hacking produce forecasted policies that score lower than counterfactuals trained without them. In \emph{Full} Box Moving, exploiting the $+5$ tile yields lower forecasted return than pursuing up-arrows (while in the \emph{No-Hack} variant, MCVL admits all transitions and recovers the Oracle policy, \autoref{fig:abl_box_notam}); in Absent Supervisor, unsupervised punishment-cell transitions produce policies that use the cell more often, scoring lower under an evaluator pretrained to always penalize it regardless of supervision state; in Rocks and Diamonds, rock-pushing transitions score below diamond-prioritizing ones.  In Tomato Watering, MCVL mitigates hacking only with a non-delusional transition model which tracks true world state and lets the evaluator distinguish genuine watering from the bucket's illusion, much as a person in VR still knows their physical location. For continuous control, sensor-corruption and reward-tampering transitions produce policies that prioritize the bonus over movement. In Ant pursuing high contact forces impedes stable policy learning for TD3, but MC-TD3 tracks Oracle.

\subsection{Ablations and sensitivity} \label{sec:sensitivity}

We study when to trigger checks, how to conduct them, and how to handle harmful transitions (\autoref{fig:ablations}). Triggering only when $|r - R_\psi(s,a)| \ge \delta_r$ (\emph{Check-by-reward}) performs comparably to \emph{Check-all} but with lower computational cost, and outperforms \emph{Discard-by-reward} (which never admits large-discrepancy transitions), as the latter filters out legitimately informative data and fails to reach optimal performance. Removing the gate entirely and bootstrapping on a frozen reward model (\emph{Freeze RM}) isolates the gate's contribution: it matches the Oracle in five of our seven tasks, so a reward model fit on the pretraining data is often already sufficient, but it collapses where that model contains errors. Notably, in \emph{Absent Supervisor}, dropping the gate induces reward hacking and a sharp fall in true performance; in one Ant seed, the no-gate policy even drove its predicted return far above any achievable true return while true performance collapsed (\autoref{apdx:freezerm}).

\paragraph{Importance of forecasting.}
An \emph{Each-step} variant that compares the policy before and after a \emph{single} gradient step does not reliably prevent hacking. Policy changes only occur once the critic begins assigning higher value to the new behavior, at which point both the critic and reward model already endorse it. By contrast, allowing $l$ standard updates during forecasting gives the base learner enough room to translate a transition into a meaningful policy shift, which the evaluator can then assess effectively using current live networks. The scoring noise for the branch comparison does not depend on $l$, while forecasted policies are more likely to diverge with $l$, so increasing $l$ improves detection reliability (see \Cref{apdx:proofs}).

\paragraph{Reject vs. penalize.}
Replacing rejected transitions with large negative rewards (\emph{Punishment}) is less effective than discarding them. Punished transitions accumulate in the buffer which discourages learning policies that exploit them. However, this also prevents learning these policies during forecasting, which decreases the reliability of the forecast-and-score gate.

\paragraph{Pretraining budget.}
As shown in \autoref{fig:init_train_steps}, some seeds avoid hacking with as few as 100 pretraining steps in \emph{Safe}; by 300 steps all seeds succeed, even though most have not converged to the optimal policy in the \emph{Safe} environment. With zero pretraining, MCVL matches the base learner.

\begin{figure*}[t!]
    \vspace{-6mm}
	\centering
	\begin{subfigure}[t]{0.46\linewidth}
		\centering
        \includegraphics[width=\linewidth,trim=2mm 0 2mm 0, clip]{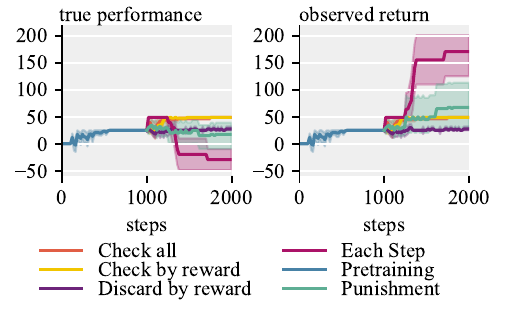}
		\vspace*{-7mm}
		\caption{Alternative training schemes}
		\label{fig:abl_box_full}
	\end{subfigure}\hfill
	\begin{subfigure}[t]{0.46\linewidth}
		\centering\includegraphics[width=\linewidth,trim=2mm 0 2mm 0, clip]{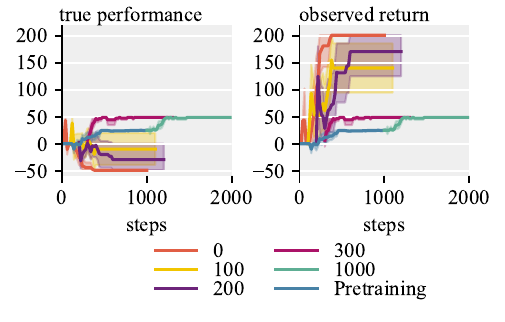}
		\vspace*{-7mm}
		\caption{Pretraining steps}
		\label{fig:init_train_steps}
	\end{subfigure}
	\caption{ Additional experiments in Box Moving. (a) Comparison of training schemes: \emph{Check all} checks all transitions; \emph{Check by reward} checks only transitions for which predicted reward differs from the observed by at least $\delta$; \emph{Discard by reward} discards all transitions where predicted reward sufficiently differs from the observed; \emph{Each step} evaluates policies before and after each gradient step without forecasting future policies; \emph{Punishment} replaces rejected transitions' rewards with a punishment reward. (b) Effect of different amounts of pretraining, 0 means no pretraining. After as low as 300 steps, MCVL can achieve optimal performance across all seeds.}
	\label{fig:ablations}
	\vspace*{-3mm}
\end{figure*}

\paragraph{Learned transition model.}
Because both branches share the same $\hat P$, exact dynamics are unnecessary; it suffices that the model preserves the relative ranking of hacking and non-hacking trajectories. Replacing the environment with a learned forward model in \emph{Box Moving} preserves Oracle performance while avoiding reward hacking (\autoref{apdx:learned-model}).

\paragraph{Forecast budget $l$.}
Too small $l$ fails to capture the policy change induced by a transition, reducing the robustness of rejecting harmful updates, which slows learning of reward hacking, but does not completely prevent it. Increasing $l$ resolves this (\autoref{fig:tam_train_steps}). Additional experiments in different variations of gridworld environments are provided in \autoref{apdx:experiments}.

\subsection{Comparison to occupancy-regularized objectives}

The closest practical baseline in standard RL settings is occupancy-regularized policy optimization toward a known safe policy~\citep{laidlaw2023preventing}. When the Oracle policy deviates substantially from the reference, the regularization strength sufficient to suppress hacking also suppresses Oracle-level improvements, so no single $\lambda$ separates the two. We test this with the objective \(F(\pi,\pi_{\text{ref}})=J(\pi)-\lambda\,D(\mu_\pi\|\mu_{\pi_{\text{ref}}})
\) holding the reference fixed to Frozen. Across 10 seeds and two divergences ($\sqrt{\chi^2}$, KL), such a $\lambda$ \emph{often does not exist} (\Cref{apdx:orpo}). By contrast, MCVL achieves performance comparable to the Oracle without relying on a safe policy.

\section{Limitations and Future Work} \label{sec:limitations}

\textbf{Computation.}
MCVL adds forecasting and scoring overhead. Triggering checks only on reward discrepancies (\Cref{sec:method}) keeps the overall wall-clock training time to $1.8\times$--$4.2\times$ slowdowns versus the base learner on continuous-control tasks (\Cref{apdx:compute}). Further reductions appear feasible through caching and batched rollouts.

\textbf{Scope of applicability.}
MCVL assumes the evaluator ranks hacking-inducing trajectories below non-hacking ones at scoring horizons. If proxy misspecification is already valued by the evaluator, harmful updates may be admitted. This may happen due to incorrect reward shaping, as in CoastRunners~\citep{OpenAI2023FaultyReward} where the agent learns to repeatedly collect boosts instead of following the track. If the reward for boosts is learned by the reward model, the gating mechanism would not reject policies that exploit it. MCVL is therefore complementary to better reward design, including potential-based shaping~\citep{Ng1999PolicyIU}.

\textbf{Transition dynamics.}
Because both branches share $\hat P$, relative ranking accuracy matters more than absolute fidelity. In \emph{Box Moving}, learned-model rollouts preserve Oracle-level MCVL performance (\Cref{apdx:learned-model}). Other environments were tested with true dynamics for scoring; learning accurate world models for higher-dimensional environments is left for future work.

\textbf{Pretraining dependence.}
MCVL needs a seed dataset with no hacking transitions to identify the intended behavior; miscalibrated initial evaluators can entrench bad preferences. In our tasks, \emph{Safe}-variant data (gridworlds) or random exploration (continuous control) was sufficient; alternatives include simpler-task pretraining, trajectory filtering, or human demonstrations. MCVL does not modify the base learner's exploration policy; it only gates which transitions enter the replay buffer, so exploration isn't affected, as demonstrated by the Oracle-comparable performance. 

\section{Related Work}\label{sec:related}
Agents exploiting misspecified objectives are studied as \emph{reward hacking}~\citep{skalse2022defining}, \emph{reward gaming}~\citep{leike2018scalable}, and \emph{specification gaming}~\citep{krakovna2020specification}. \citet{krakovna2020specification} survey such failures, and \citet{skalse2022defining} formalize reward hacking and show that it cannot be prevented without restricting the set of possible policies or controlling optimization.

A common mitigation is constraining policy to remain close to a trusted behavior distribution~\citep{laidlaw2023preventing, liu2026robust}. MCVL instead filters individual replay updates using counterfactual forecasting, without requiring a safe reference policy or specific proxy-family assumptions. In our settings, MCVL reaches optimal returns where ORPO-style objectives often cannot simultaneously avoid hacking and learn the optimal policy (\autoref{apdx:orpo}).

Direct manipulation of reward channels is studied as \emph{wireheading}~\citep{amodei2016concrete, taylor2016alignment, everitt2016avoiding, majha2019categorizing} or \emph{reward tampering}~\citep{kumar2020realab, everitt2021reward}. A long-standing idea is \emph{current utility optimization}: choose actions that improve the current objective without changing what is optimized~\citep{yudkowsky2011complex, hibbard2012model, yampolskiy2014utility}. \citet{schmidhuber2003godel} describes a self-modifying \emph{G\"odel machine} agent that adopts only code or utility changes provably beneficial according to the current objective. \citet{everitt2016avoiding} consider Bayesian agents over hand-specified utility functions that select actions to avoid altering beliefs about the reward mechanism, and \citet{everitt2021reward} give conditions under which optimizing the \emph{current} reward avoids incentives to tamper. MCVL operationalizes this perspective in practical off-policy value-based RL, including settings beyond direct reward/sensor tampering.

\section{Conclusion} \label{sec:discussion}
We introduced \emph{Modification-Considering Value Learning}, a forecast-and-score safeguard for off-policy value-based RL. MCVL evaluates two counterfactual update paths with a frozen bootstrapped-return estimator and admits a transition only if inclusion does not reduce that score, optimizing what the agent currently values while remaining conservative about changing those values.

Implementations, MC-DDQN and MC-TD3, mitigate reward hacking in safety-relevant gridworlds and modified MuJoCo tasks while remaining close to Oracle on true reward. By operationalizing ideas from current utility optimization within standard deep RL, MCVL offers a practical way toward avoiding reward hacking without sacrificing performance on the intended objective.

\subsubsection*{Acknowledgments}
Resources used in preparing this research were provided, in part, by the Province of Ontario, the Government of Canada through CIFAR, companies sponsoring the Vector Institute, Trajectory Labs (\href{https://trajectorylabs.org}{trajectorylabs.org}) and the Digital Research Alliance of Canada (\href{https://alliancecan.ca}{alliancecan.ca}). Any opinions, findings, conclusions, or recommendations expressed in this material are those of the authors and do not necessarily reflect the view of the Canadian government.

\appendix

\section{Proofs and Error Decomposition}\label{apdx:proofs}

\begin{proof}[Proof of \Cref{prop:safety-permissiveness}]
All three claims follow from the $\epsilon$-accuracy of $\hat J$.

\emph{(i) Safety.} By $\epsilon$-accuracy, $\hat J(\tilde\pi^{\,+}) \le J_{R^*}(\tilde\pi^{\,+}) + \epsilon$ and $\hat J(\tilde\pi^{\,0}) \ge J_{R^*}(\tilde\pi^{\,0}) - \epsilon$. Combining with the hypothesis $J_{R^*}(\tilde\pi^{\,+}) < J_{R^*}(\tilde\pi^{\,0}) - 2\epsilon$:
\[
\hat J(\tilde\pi^{\,+}) \le J_{R^*}(\tilde\pi^{\,+}) + \epsilon < J_{R^*}(\tilde\pi^{\,0}) - 2\epsilon + \epsilon = J_{R^*}(\tilde\pi^{\,0}) - \epsilon \le \hat J(\tilde\pi^{\,0}).
\]
Hence $\hat J(\tilde\pi^{\,+}) < \hat J(\tilde\pi^{\,0})$ and the accept condition fails.

\emph{(ii) Permissiveness.} MCVL rejects means $\hat J(\tilde\pi^{\,+}) < \hat J(\tilde\pi^{\,0})$. By $\epsilon$-accuracy:
\[
J_{R^*}(\tilde\pi^{\,+}) \le \hat J(\tilde\pi^{\,+}) + \epsilon < \hat J(\tilde\pi^{\,0}) + \epsilon \le J_{R^*}(\tilde\pi^{\,0}) + 2\epsilon.
\]
Hence $J_{R^*}(\tilde\pi^{\,+}) < J_{R^*}(\tilde\pi^{\,0}) + 2\epsilon$, so no transition with true improvement exceeding $2\epsilon$ is rejected.

\emph{(iii) Bounded degradation.} Acceptance implies $\hat J(\tilde\pi^{\,+}) \ge \hat J(\tilde\pi^{\,0})$. By $\epsilon$-accuracy:
\[
J_{R^*}(\tilde\pi^{\,+}) \ge \hat J(\tilde\pi^{\,+}) - \epsilon \ge \hat J(\tilde\pi^{\,0}) - \epsilon \ge J_{R^*}(\tilde\pi^{\,0}) - 2\epsilon. \qedhere
\]
\end{proof}

\paragraph{Error decomposition.}
When rollouts use the true dynamics ($\hat P = P$) and the number of rollouts $k \to \infty$ (or the environment is deterministic), the evaluator error decomposes as follows.

\begin{lemma}[Evaluator error bound]\label{lem:error-decomp}
Under the conditions above,
\[
|\hat J(\pi) - J_{R^*}(\pi)| \;\le\; \frac{1 - \gamma^n}{1 - \gamma}\,\epsilon_R \;+\; \gamma^n\,\epsilon_Q,
\]
where $\epsilon_R = \max_{s,a} |R_\psi(s,a) - R^*(s,a)|$ and $\epsilon_Q = \max_{s,a} |Q_\theta(s,a) - Q^\pi_{R^*}(s,a)|$, with $Q^\pi_{R^*}$ denoting the true action-value function of $\pi$ under $R^*$.
\end{lemma}

\begin{proof}
Under exact dynamics and expectations,
\begin{align*}
\hat J(\pi) - J_{R^*}(\pi) &= \mathbb{E}_{\rho,\pi}\!\left[\sum_{t=0}^{n-1} \gamma^t \bigl(R_\psi(s_t,a_t) - R^*(s_t,a_t)\bigr) \;+\; \gamma^n \bigl(Q_\theta(s_n,a_n) - Q^\pi_{R^*}(s_n,a_n)\bigr)\right].
\end{align*}
Taking absolute values and applying the triangle inequality,
\[
|\hat J(\pi) - J_{R^*}(\pi)| \;\le\; \sum_{t=0}^{n-1}\gamma^t\,\epsilon_R + \gamma^n\,\epsilon_Q \;=\; \frac{1-\gamma^n}{1-\gamma}\,\epsilon_R + \gamma^n\,\epsilon_Q. \qedhere
\]
\end{proof}

Substituting into \Cref{prop:safety-permissiveness}, MCVL rejects a hacking-inducing transition whenever the true return drop $\delta = J_{R^*}(\tilde\pi^{\,0}) - J_{R^*}(\tilde\pi^{\,+})$ satisfies
\begin{equation}\label{eq:sufficient}
\delta \;>\; 2\!\left(\frac{1-\gamma^n}{1-\gamma}\,\epsilon_R + \gamma^n\,\epsilon_Q\right).
\end{equation}
This highlights the role of the scoring rollout steps~$n$: each additional step replaces reliance on the bootstrap $Q_\theta$ with a direct reward-model observation. The bound $g(n) = \frac{1-\gamma^n}{1-\gamma}\,\epsilon_R + \gamma^n\,\epsilon_Q$ is monotonically decreasing in $n$ when $\epsilon_Q > \frac{\epsilon_R}{1-\gamma}$, which we expect in practice since $R_\psi$ directly observes the forecasted policy's behavior during rollouts while $Q_\theta$ was not trained on the forecasted policy. When $\hat P$ is a learned model, compounding prediction error provides a practical upper bound on useful $n$; in episodic tasks with a simulator, the episode length is a natural choice. Successful filtering keeps the buffer free of hacking transitions, helping maintain small $\epsilon_R$ and $\epsilon_Q$ over time.

\paragraph{Why forecasting improves detection.}
MCVL scores each branch with $k$ rollouts of length $n$ under the frozen evaluator. Because evaluator parameters are frozen and $n$ is fixed, the variance of the per-rollout bootstrapped return is bounded by some $\sigma^2$ independently of the forecast budget $l$. Using paired rollouts with shared start states exploits positive covariance between the two branch estimates, so the variance of the estimator $\hat J_k(\tilde\pi^{\,0}) - \hat J_k(\tilde\pi^{\,+})$ is at most $2\sigma^2/k$, independent of $l$.

The expected scoring gap $\Delta(l) = \hat J(\tilde\pi^{\,0}_l) - \hat J(\tilde\pi^{\,+}_l)$ (positive when the hacking branch scores lower) in general grows with $l$: each additional forecast step gives the base learner more room to translate the hacking transition into a policy change the evaluator can distinguish. So increasing the forecast budget improves detection without increasing scoring cost.

At $l=1$ (the \emph{Each-step} variant), a single gradient step barely changes the policy, especially before the critic has begun to assign high value to the hacking behavior, so $\Delta(1) \approx 0$ and the check is unreliable. Increasing $l$ to the point where the hacking transition induces a meaningful policy divergence produces a gap the evaluator can reliably detect, explaining why the Each-step variant fails while moderate $l$ succeeds (\autoref{fig:tam_train_steps}).

\section{Feasibility of Occupancy-Regularized Objectives}
\label{apdx:orpo}
It would be trivial to show that regularizing to a safe policy either performs at the same level as the frozen safe policy (or reward hacks) by selecting a high (or low) regularization coefficient. Instead, we test whether an ORPO-style objective presented by \citep{laidlaw2023preventing} could, \emph{in principle}, select the desired behavior in our settings. For each gridworld environment we train DDQN Q-functions for \emph{Frozen} (safe, post-pretraining), \emph{Hacking} (trained on observed reward), and \emph{Oracle} (trained on true reward). From these Q-functions, we derive stochastic policies via (i) softmax over Q-values and (ii) $\epsilon$-greedy with $\epsilon=0.05$. We estimate occupancy measures and empirical on-policy discounted episodic returns under observed reward $J_\pi$ with 1000 rollouts. We compute ORPO objective $F(\pi,\pi_{\mathrm{Frozen}})=J(\pi)-\lambda D(\mu_{\pi}\Vert \mu_{\pi_{\mathrm{Frozen}}})$ for $D\in\{\mathrm{KL},\sqrt{\chi^2}\}$. We record the fraction of seeds (out of 10) where some $\lambda>0$ exists such that it satisfies \emph{both} $F(\pi_{\mathrm{Oracle}},\pi_{\mathrm{Frozen}}) > F(\pi_{\mathrm{Frozen}},\pi_{\mathrm{Frozen}})$ and $F(\pi_{\mathrm{Oracle}},\pi_{\mathrm{Frozen}}) > F(\pi_{\mathrm{Hacking}},\pi_{\mathrm{Frozen}})$. Let $J_O,J_F,J_H$ denote the observed returns of Oracle, Frozen, and Hacking policies; let $D_O=D(\mu_O\Vert\mu_F)$ and $D_H=D(\mu_H\Vert\mu_F)$. The first inequality gives $\lambda < \frac{J_O-J_F}{D_O}$ when $D_O>0$ (and requires $J_O>J_F$ when $D_O=0$). The second inequality is $\lambda(D_H-D_O) > J_H-J_O$, which yields three cases: if $D_H>D_O$, then $\lambda > \frac{J_H-J_O}{D_H-D_O}$; if $D_H<D_O$, then $\lambda < \frac{J_H-J_O}{D_H-D_O}$; if $D_H=D_O$, it requires $J_O>J_H$. A seed is feasible iff the resulting open interval intersects $(0,+\infty)$. Per-seed feasibility does not imply a single global $\lambda$ across seeds. We present results in \autoref{table:orpo}.

\begin{table}[h]
\centering
\small
\begin{tabular}{lcccccc}
\toprule
Policy & Divergence & Box Moving & Absent Supervisor & Tomato Watering & Rocks \& Diamonds \\
\midrule
Soft-Q         & $\sqrt{\chi^2}$ & 0\% & 100\% & 0\% & 0\% \\
Soft-Q         & KL       & 0\% & 100\% & 0\% & 0\% \\
$\epsilon$-greedy & $\sqrt{\chi^2}$ & 60\% & 40\% & 30\% & 0\% \\
$\epsilon$-greedy & KL       & 40\% & 50\% & 0\%  & 0\% \\
\bottomrule
\end{tabular}
\caption{Percentage of seeds (of 10) where a regularization weight $\lambda>0$ exists that ranks the Oracle policy above both Frozen and Hacking under an ORPO-like objective.}
\label{table:orpo}
\end{table}

In many cases, no such $\lambda$ exists, suggesting that occupancy regularization fails to suppress high-value hacks without also suppressing Oracle-like improvements. In contrast, MCVL achieves performance comparable to the Oracle across all these tasks.

\newpage
\bibliography{main}
\bibliographystyle{rlj}

\beginSupplementaryMaterials

\section{Additional Experiments} \label{apdx:experiments}
\begin{figure}[h]
	\centering
        \begin{subfigure}[t]{.49\linewidth}
		\centering\includegraphics[width=1\linewidth]{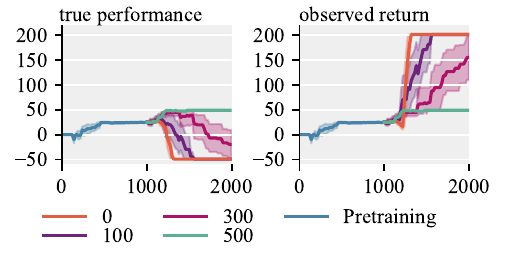}
		\vspace*{-7mm}
		\caption{Forecasting training steps}
		\label{fig:tam_train_steps}
	\end{subfigure}\hfill
	\begin{subfigure}[t]{.49\linewidth}
		\centering\includegraphics[width=1\linewidth]{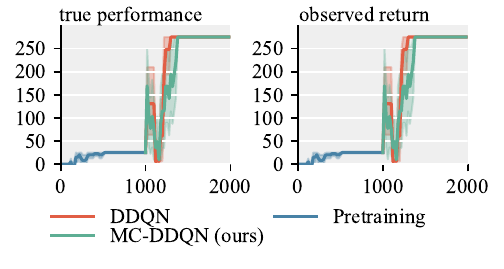}
		\vspace*{-7mm}
		\caption{Box Moving \emph{No-Hack}}
		\label{fig:abl_box_notam}
	\end{subfigure}
 
	\begin{subfigure}[t]{.49\linewidth}
		\centering\includegraphics[width=1\linewidth]{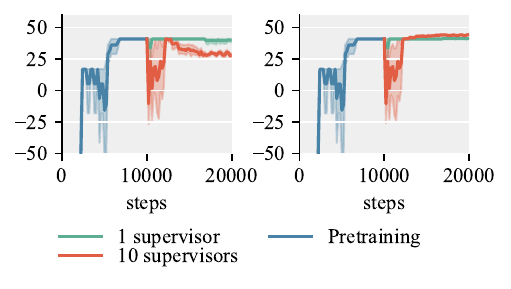}
		\vspace*{-7mm}
		\caption{Number of supervisors}
		\label{fig:num_supervisors}
	\end{subfigure} \hfill
	\begin{subfigure}[t]{.49\linewidth}
		\centering\includegraphics[width=1\linewidth]{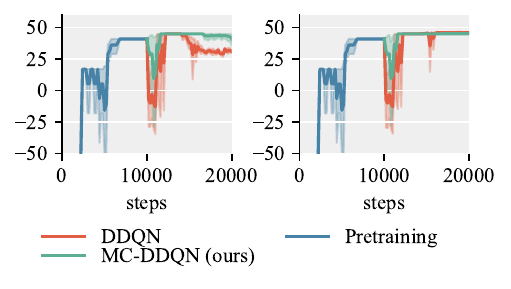}
		\vspace*{-7mm}
		\caption{Removing walls in \emph{Absent Supervisor}}
		\label{fig:no_walls}
	\end{subfigure} \hfill
	\caption{(a) Sensitivity to forecasting training steps $l$ in Box Moving. (b) Results in the \emph{No-Hack} version of Box Moving. (c) Varying the number of supervisors in Absent Supervisor. (d) A variant of Absent Supervisor where a shorter path becomes available in \emph{Full}. }
	\label{fig:addit-exps}
\end{figure}

In \autoref{fig:tam_train_steps}, we investigated the number of forecasting training steps $l$ needed to avoid undesired behavior in Box Moving. With an insufficient number of training steps, certain undesired transitions are not rejected, yet our algorithm still slows down the learning of reward hacking behavior.

In \autoref{fig:abl_box_notam}, we examine the behavior of MC-DDQN in the \emph{No-Hack} version of \emph{Box Moving} (\autoref{fig:box}). In this version, the agent receives a +5 reward on the top cell which does not interfere with moving the box upward. As anticipated, in this scenario our agent does not reject transitions and learns the optimal policy.

We also conducted experiments in \emph{Absent Supervisor}, varying the number of supervisors. In \autoref{fig:num_supervisors}, increasing the number of supervisors from 1 to 10 leads to less consistent detection of transitions that induce reward hacking, despite the change being purely visual. Qualitative analysis revealed that our neural networks struggled to adapt to this distribution shift (all gridworld environments are encoded as multi-hot vectors), resulting in predicted rewards deviating significantly from the ground truth.

Furthermore, we explored the impact of removing two walls from \emph{Absent Supervisor} after training in \emph{Safe}. Without these two walls, a shorter path to the goal is available that bypasses the punishment cell, although going through the punishment cell remains faster. In \autoref{fig:no_walls}, it is evident that while our algorithm can learn a better policy that avoids the punishment cell, the rejection of reward hacking transitions becomes less reliable. This decline is attributed to the increased distribution shift between \emph{Safe} and \emph{Full}.

\section{Learned Transition Model}
\label{apdx:learned-model}
\begin{wrapfigure}{r}{0.5\textwidth}
    \centering
    \includegraphics[width=\linewidth]{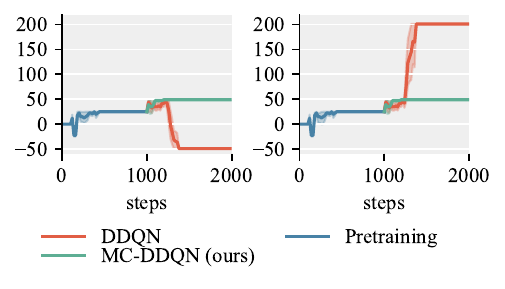}
    \caption{MC-DDQN with a learned transition model in \emph{Box Moving}.}
    \label{fig:learned-model}
\end{wrapfigure}

Scoring uses a transition model $\hat P$ solely to \emph{compare} two candidate policies under a frozen evaluator; exact dynamics are unnecessary as long as the evaluator continues to rank hacking trajectories below non-hacking ones. To verify that a learned model suffices, we train a forward dynamics model and use it in place of the environment during scoring rollouts. The model is a two-hidden-layer MLP with 128 units per layer, ReLU activations, and layer normalization. It takes the concatenation of the current observation and a one-hot encoded action as input and predicts the next observation. The model is trained with MSE loss using Adam (learning rate $10^{-2}$) on 50 episodes of random exploration data (1000 gradient steps, batch size 256) and frozen during deployment. We run MC-DDQN in \emph{Box Moving} with the same hyperparameters as in the main experiments, replacing only the transition source. As shown in \autoref{fig:learned-model}, MC-DDQN with the learned $\hat P$ avoids reward hacking and reaches Oracle performance, matching the result obtained with the true environment. This supports the claim that approximate dynamics suffice for reliable gating.

\section{Implementation Details of MC-DDQN} \label{apdx:full-alg}
\begin{algorithm}[H]
    \caption{Policy Forecasting} \label{alg:policy_forecasting} 
    \textbf{Input}: Set of transitions $T$, replay buffer $D$, current Q-network parameters $\theta$, training steps $l$ \\
    \textbf{Output}: Forecasted policy $\pi_f$
    \begin{algorithmic}[1]
        \State $\theta_f \leftarrow \Call{Copy}{\theta}$ \Comment{Copy current Q-network parameters}
        \For{training step $t = 1$ to $l$}
            \State Sample random mini-batch $B$ of transitions from $D$
            \State $\theta_f \leftarrow \Call{TrainDDQN}{\theta_f, B \cup T}$ \Comment{$T$ is added to each batch for deterministic environments}
        \EndFor
        \State \Return $\pi_f(s) = \argmax_a Q_{\theta_f}(s, a)$ \Comment{Return forecasted policy}
    \end{algorithmic}
\end{algorithm}

\begin{algorithm}[H]
    \caption{Scoring} \label{alg:utility_estimation}
    \textbf{Input}: Policy $\pi$, transition model $\hat P$, return estimator parameters $\theta$ and $\psi$, initial states $\rho$, rollout steps $n$, number of rollouts $k$ \\
    \textbf{Output}: Estimated bootstrapped return of the policy $\pi$
    \begin{algorithmic}[1]
        \For{rollout $r = 1$ to $k$}
            \State $g \leftarrow 0$ \Comment{Initialize return for this rollout}
            \State $s_0 \sim \rho$ \Comment{Sample an initial state}
            \State $a_0 \leftarrow \pi(s_0)$ \Comment{Get action from policy}
            \For{step $t = 0$ to $n-1$}
                \State $g \leftarrow g + \gamma^t R_\psi(s_t, a_t) $ \Comment{Accumulate predicted reward}
                \State $s_{t+1} \sim \hat P(s_t, a_t)$ \Comment{Sample next state from transition model}
                \State $a_{t+1} \leftarrow \pi(s_{t+1})$  \Comment{Get action from policy}
            \EndFor
            \State $g \leftarrow g + \gamma^nQ_\theta(s_n, a_n)$ \Comment{Add final Q-value}
        \EndFor
        \State \Return $\frac{1}{k} \sum_{r=1}^{k} g$ \Comment{Return average return over rollouts}
    \end{algorithmic}
\end{algorithm}

\begin{algorithm}[H]
    \caption{Modification-Considering Double Deep Q-learning (MC-DDQN)} \label{alg:mc_ddqn}
    \textbf{Input}: Pretrained return estimator parameters $\theta$ and $\psi$, replay buffer $D$, transition model $\hat P$, initial states $\rho$, rollout steps $n$, number of rollouts $k$, forecasting train steps $l$, threshold $\delta_r$. \\
    \textbf{Output}: Trained Q-network and reward model
    \begin{algorithmic}[1]
        \State Initialize state $s_0$
        \For{time step $t = 0$ to end of training}
            \State $a_t \leftarrow$ \Call{$\epsilon$-greedy}{$\argmax_a Q_\theta(s_t, a)$}
            \State Execute action $a_t$, observe reward $r_t$, and transition to state $s_{t+1}$
            \State $T_t \leftarrow (s_t, a_t, r_t, s_{t+1})$
            \If{$\lvert r_t - R_\psi(s_t,a_t)\rvert < \delta_r$}
                \State $\mathit{accept} \leftarrow \textbf{True}$
            \Else
                \State $\tilde\pi^{\,+} \leftarrow$ \Call{Forecast}{$\{T_{t}\}, D, \theta, l$} \Comment{Forecast a policy with new transition}
                \State $\tilde\pi^{\,0} \leftarrow$ \Call{Forecast}{$\{\}, D, \theta, l$} \Comment{Forecast a policy without new transition}
                \State $J_{\tilde\pi^{\,+}} \leftarrow$ \Call{Score}{$\tilde\pi^{\,+}, \hat P, \theta, \psi, \rho, n, k$} \Comment{Estimate n-step bootstrapped return for $\tilde\pi^{\,+}$}
                \State $J_{\tilde\pi^{\,0}} \leftarrow$ \Call{Score}{$\tilde\pi^{\,0}, \hat P, \theta, \psi, \rho, n, k$} \Comment{Estimate n-step bootstrapped return for $\tilde\pi^{\,0}$}
                \State $\mathit{accept} \leftarrow (J_{\tilde\pi^{\,+}} \geq J_{\tilde\pi^{\,0}})$ \Comment{Accept if $\tilde\pi^{\,+}$ is not worse by current estimator}
            \EndIf
            \If{$\mathit{accept}$}
                \State Store transition $T_t$ in $D$
            \Else
                \State Reset the environment \Comment{Without $T_t$ in $D$ future forcasted policies might fail to learn.}
            \EndIf
            \State Sample random mini-batch $B$ of transitions from $D$
            \State $\theta \leftarrow$ \Call{TrainDDQN}{$\theta, B$} \Comment{Update Q-network}
            \State $\psi \leftarrow$ \Call{Train}{$\psi, B$} \Comment{Update reward model using $L_2$ loss}
            \State $s_t \leftarrow s_{t+1}$
        \EndFor
    \end{algorithmic}
\end{algorithm}

\section{Implementation Details of MC-TD3} \label{apdx:mctd3}
Our implementation is based on the implementation provided by~\citet{huang2022cleanrl}. The overall structure of the algorithm is consistent with MC-DDQN, described in \autoref{apdx:full-alg}, with key differences outlined below. TD3 is an actor-critic algorithm, meaning that the parameters \(\theta\) define both a policy (actor) and a Q-function (critic). In \autoref{alg:policy_forecasting} and \autoref{alg:mc_ddqn}, calls to \textsc{TrainDDQN} are replaced with \textsc{TrainTD3}, which updates the actor and critic parameters \(\theta\) as specified by \citet{fujimoto2018addressing}. Additionally, in \autoref{alg:policy_forecasting}, the returned policy \(\pi_f(s)\) corresponds to the actor rather than \(\arg\max_a Q_\theta(s, a)\), and in \autoref{alg:mc_ddqn} the action executed in the environment is also selected by the actor.

\paragraph{Forecast implementation detail.}
In \autoref{alg:policy_forecasting}, the transition set $T$ is added to each forecast minibatch ($B \cup T$) rather than relying on stochastic replay sampling from $D \cup T$. This is an implementation choice to reduce variance in the branch comparison by ensuring the checked transition is represented in the ``with-$T$'' branch under a fixed update budget. The theoretical analysis is unchanged, because it only assumes an abstract $\mathsf{Forecast}(\cdot,l)$ operator.

\section{Details of the Continuous-Control Experiments} \label{apdx:reacher}

All three continuous-control environments follow the standard procedure where the policy and critics are first trained on transitions collected by a random policy. No \emph{Safe} variant is required because random exploration is unlikely to simultaneously satisfy the velocity and physical-configuration conditions that trigger hacking or keep the arm in place for multiple frames in Reacher. This demonstrates that MCVL can function effectively without a separate safe environment whenever the hacking affordance is difficult to discover by chance.

\paragraph{Reacher.}
The rewards in the original Reacher-v5 environment are calculated as the sum of the negative distance to the target and the negative joint actuation strength. This reward structure encourages the robotic arm to reach the target while minimizing large, energy-intensive actions. The target's position is randomized at the start of each episode, and random noise is added to the joint rotations and velocities. Observations include the angles and angular velocities of each joint, the target's coordinates, and the difference between the target's coordinates and the coordinates of the arm's end. Actions consist of torques applied to the joints, and each episode is truncated after 50 steps.

We modified the environment by introducing a +50 reward when the arm's end remains within a small, fixed region for 15 consecutive steps. This region remains unchanged across episodes, simulating a scenario where the robot can tamper with its reward function by holding a button, but such behavior is difficult to discover. In our setup, a reward-tampering policy is highly unlikely to emerge through random actions and is typically discovered only when the target happens to be inside the reward-tampering region.

In our experiments, TD3 consistently discovers the hack when the goal in one of the episodes falls within the hacking region. MC-TD3 successfully filters out such transitions.

\paragraph{Ant.}
The standard Ant-v5 environment rewards forward movement (velocity in the $x$-direction), a healthy-posture bonus, and penalizes control effort and contact forces. We augment it with a sensor-malfunction model motivated by real-world IMU failures in legged robots: when the total contact-force norm exceeds a threshold (400) and the agent's forward velocity is above a minimum (3.5), the reward is augmented with a bonus equal to 10 times the contact force. This models the self-reinforcing loop in which aggressive gaits produce high ground-reaction forces that saturate accelerometers, causing inflated velocity readings that feed back into the reward. The performance metric is the reward \emph{before} the sensor-corruption bonus is applied. Random exploration is unlikely to simultaneously achieve the required high velocity and high contact forces.

In our experiments TD3 learner with this bonus reward achieved lower evaluation return compared to Oracle that did not receive the bonus reward (\Cref{fig:main-ant}). We hypothesize that pursuing transitions that achieved high contact forces in the past is brittle and doesn't lead to stable policies.

\paragraph{HalfCheetah.}
The standard HalfCheetah-v5 environment rewards forward velocity and penalizes control effort. We augment it with an optical-flow sensor failure model: when the front tip of the cheetah drops below a height threshold (-0.23) while the forward velocity exceeds a minimum (6), 1 reward bonus is added for each millimeter of the tip's negative height. This models the velocity over-estimation that occurs when a downward-facing optical-flow sensor operates too close to the ground, where proximity saturation and height-division error compound. The performance metric excludes this bonus. Random exploration rarely produces the combination of high forward speed and low front-tip height needed to trigger the hack.

In HalfCheetah, MC-TD3's mean falls slightly below the Oracle though within the 95\% CI, likely because of occasional over-rejection of legitimate transitions or due to environment resets when hacking is discovered by exploration.

\section{Frozen Reward Model Baseline} \label{apdx:freezerm}

To isolate the contribution of the forecast-and-score gate from that of merely having a learned reward model, we evaluate a \emph{Freeze RM} baseline. Unlike the \emph{Frozen} baseline of \autoref{fig:main-results}, which stops policy and critic learning after pretraining, here only the reward model is frozen: the policy and critic continue to train, but bootstrap on the frozen pretrained reward model $R_\psi$ instead of the observed proxy reward, and \emph{every} transition is admitted (no gate). The reward model is frozen at the end of pretraining (for continuous control, at the switch from random exploration to the learned policy). All other components (pretrained initialization, replay buffer, and hyperparameters) are identical to MCVL (MC-DDQN in the gridworlds, MC-TD3 in continuous control). Freezing $R_\psi$ is necessary because, without a gate, a live reward model would simply fit the hacking bonus and the baseline would collapse to the base learner.

\begin{figure*}[t]
	\centering
    \captionsetup[subfigure]{justification=centering}
	\begin{subfigure}[t]{.142\linewidth}
		\centering\includegraphics[width=1\textwidth,trim=2mm 0mm 2mm 0mm, clip]{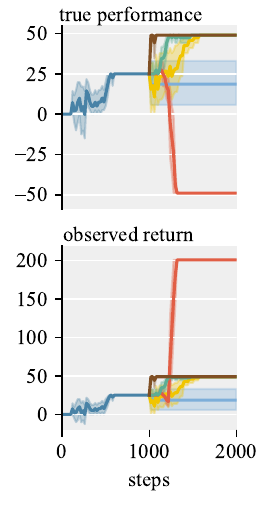}
		\vspace*{-6mm}
		\caption{Box \\ Moving}
	\end{subfigure}\hfill
    \hspace{-4mm}
	\begin{subfigure}[t]{.145\linewidth}
		\centering\includegraphics[width=1\textwidth,trim=2mm 0mm 2mm 0mm, clip]{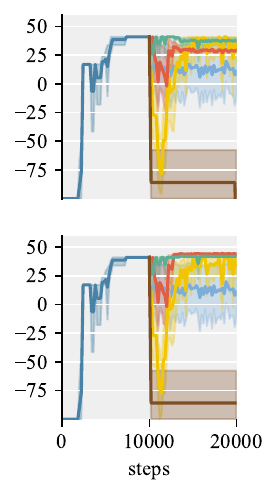}
		\vspace*{-6mm}
		\caption{Absent \\ Supervisor}
	\end{subfigure}\hfill
    \hspace{-4mm}
	\begin{subfigure}[t]{.138\linewidth}
		\centering\includegraphics[width=1\textwidth,trim=2mm 0mm 2mm 0mm, clip]{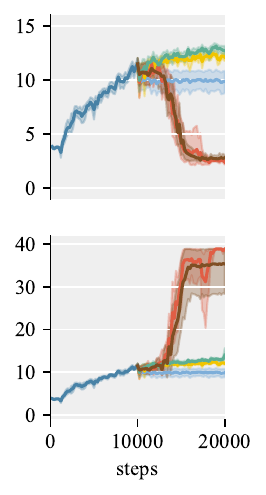}
		\vspace*{-6mm}
		\caption{Tomato \\ Watering}
	\end{subfigure}\hfill
    \hspace{-4mm}
	\begin{subfigure}[t]{.150\linewidth}
		\centering\includegraphics[width=1\textwidth,trim=2mm 0mm 2mm 0mm, clip]{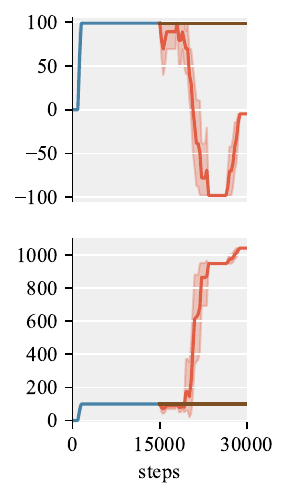}
		\vspace*{-6mm}
		\caption{Rocks \& \\ Diamonds}
	\end{subfigure}\hfill
    \hspace{-4mm}
	\begin{subfigure}[t]{.148\linewidth}
		\centering\includegraphics[width=1\textwidth,trim=0mm 0mm 0mm 0mm, clip]{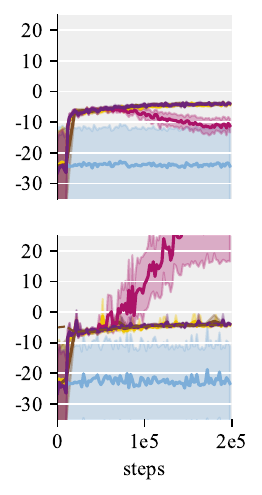}
		\vspace*{-6mm}
		\caption{Reacher}
	\end{subfigure}\hfill
    \hspace{-4mm}
	\begin{subfigure}[t]{.153\linewidth}
		\centering\includegraphics[width=1\textwidth,trim=0mm 0mm 0mm 0mm, clip]{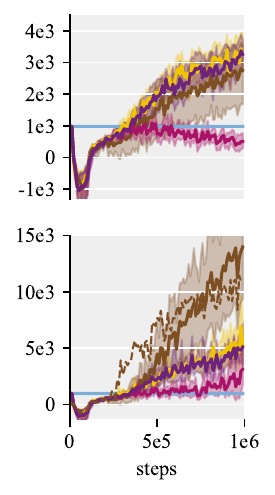}
		\vspace*{-6mm}
		\caption{Ant}
	\end{subfigure}\hfill
        \hspace{-4mm}
	\begin{subfigure}[t]{.154\linewidth}
		\centering\includegraphics[width=1\textwidth,trim=0mm 0mm 0mm 0mm, clip]{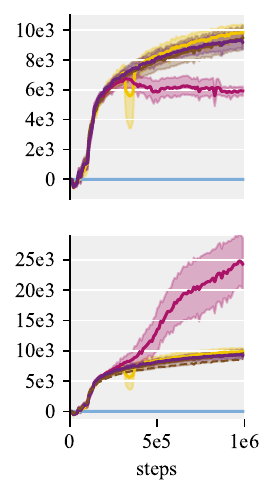}
		\vspace*{-6mm}
		\caption{HalfCheetah}
	\end{subfigure}
	\begin{subfigure}[t]{1\linewidth}
		\centering\includegraphics[width=0.95\linewidth,trim=0mm 0mm 0mm 0mm, clip]{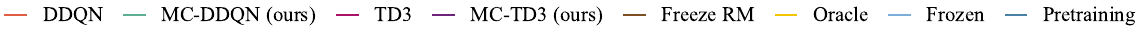}
	\end{subfigure}
	\caption{\textbf{Freeze RM baseline across all environments.} Top: true performance (unobserved objective); bottom: observed return (proxy). Same format as \autoref{fig:main-results}, with the Freeze RM baseline (brown) added. The dashed brown line in the continuous-control panels is the reward model's own predicted return for Freeze RM (the quantity it maximizes). Bold lines: mean over 10 seeds; bands: bootstrapped 95\% CI.}
	\label{fig:freezerm}
\end{figure*}

\autoref{fig:freezerm} reports Freeze RM against MCVL and the Oracle over 10 seeds per environment. On the \emph{true} objective, Freeze RM matches the Oracle in five of the seven tasks (Box Moving, Rocks \& Diamonds, Reacher, Ant, and HalfCheetah): freezing the reward model immediately after pretraining is, perhaps surprisingly, enough to recover Oracle-level performance. The two exceptions show that the gate can offer an advantage in some environments. In \emph{Absent Supervisor}, Freeze RM's true performance collapses far below the Oracle (paired $t$-test vs.\ MC-DDQN, $p<10^{-4}$; $9/10$ seeds fall below the worst Oracle seed), whereas MC-DDQN stays at Oracle level. The observed return also stays low, which indicates that the reward model is not accurately capturing the true reward. In \emph{Tomato Watering}, Freeze RM exhibits the classic hacking signature, with the observed return inflating well above the achievable true return while true performance drops; however, MCVL's mitigation in this environment relies on a non-delusional transition model used by the gate, and since Freeze RM has no gate, the comparison conflates the gate with the world model and cannot be made fair.

Within continuous control, we observe a high proxy return for Freeze RM in the Ant environment. It is significantly higher than MC-TD3's (paired $t$-test, $p<0.01$; ${\approx}2.7\times$), and its reward-model-predicted return (dashed) tracks the inflated proxy rather than the true performance. In the single seed where this is most severe, the policy exploits errors in the frozen $R_\psi$, driving the predicted return far above any achievable true return while true performance collapses; because only one seed is affected, the effect on mean true performance is small and not significant. Freeze RM is also somewhat more failure-prone than MCVL, with more seeds whose converged true performance falls below the worst Oracle seed (Ant $3$ vs.\ $2$, HalfCheetah $3$ vs.\ $1$).

Overall, the gate's necessity is environment-dependent. Where a reward model fit on the pretraining data already captures the true reward, bootstrapping on it without a gate suffices; where it does not, as in \emph{Absent Supervisor}, removing the gate leads to reward hacking and collapse. We therefore expect the gate's contribution to grow in settings where a small pretraining set cannot capture the true reward.

\section{Qualitative Observations} \label{apdx:qualitative}
During preliminary experiments, we encountered instances where the algorithm failed to reject transitions that induce reward hacking. Here we describe these occurrences and how they can be addressed.

\paragraph{Return estimation rollout steps.} When using much smaller rollout steps $n$, we noticed that during evaluation of forecasted trajectories, the non-hacking policy sometimes needed to traverse several states with low rewards to reach a high-reward region. In such cases, the reward hacking policy, which remained stationary, had a higher estimated utility. Increasing $n$ resolved this issue.

\paragraph{Forecasting without a counterfactual.} Initially, we forecasted only one future policy by training with the checked transition added to each mini-batch, and compared the resulting policy to the current one. However, in some cases this led to situations where the copy learned better non-hacking behaviors than the current policy simply because it was trained for longer. The solution was to forecast two policies, one with the checked transition added to each mini-batch and one without.

\paragraph{Sensitivity to stochasticity.} Evaluations in stochastic environments were noisy. To mitigate this, we compared the two policies starting from the same set of states and using the same random seeds of the transition model. We also kept the random seeds fixed while sampling mini-batches.

\paragraph{Handling rejected transitions.} We observed that if a hacking-inducing transition was removed from the replay buffer and another such transition occurred in the same episode, the algorithm sometimes failed to detect it the second time because there was no set of transitions in the buffer connecting this second transition to the starting state. We therefore reset the environment on every rejection, as shown in \autoref{alg:mcvl}. In practical applications, it would be reasonable to assume that after detecting potential reward hacking, the agent would be returned to a safe state instead of continuing exploration. Alternatively, the learning can be just disabled until the end of the episode.
 
\paragraph{Irreversible changes.} In \emph{Rocks and Diamonds}, when comparing policies starting from the current state after the rock was pushed into the goal area, the comparison results were always the same, as it was impossible to move the rock out of the goal area. We addressed this by evaluating from the initial state of the environment. In cases where reset is not possible, the agent may store starting states in a buffer. This issue underscores the importance of future research into avoiding irreversible changes.

\section{Computational Requirements} \label{apdx:compute}
All gridworld experiments were conducted on workstations equipped with Intel\textsuperscript{\textregistered} Core\texttrademark i9-13900K processors and NVIDIA\textsuperscript{\textregistered} GeForce RTX\texttrademark4090 GPUs. Continuous control experiments were conducted on NVIDIA\textsuperscript{\textregistered} H100 MIG partitions (1/8 GPU, 8 vCPU cores). All experiments in the \emph{Absent Supervisor} and \emph{Tomato Watering} environments each required 12-14 GPU-hours, running 10 seeds in parallel. In \emph{Rocks and Diamonds}, experiments took 1 GPU-day, while in \emph{Box Moving} they required 2 hours each. For the continuous-control environments (1M training steps for Ant and HalfCheetah, 200k for Reacher), average per-seed wall-clock times on an H100 MIG partition were: Reacher TD3 75 min / MC-TD3 135 min ($1.8\times$), Ant TD3 68 min / MC-TD3 206 min ($3.0\times$), HalfCheetah TD3 65 min / MC-TD3 272 min ($4.2\times$). The variation in overhead across environments depends primarily on how frequently the reward-discrepancy check is triggered. In total, the main experiments described in \Cref{sec:experiments} required approximately 20 GPU-days, including around 6 GPU-days for baselines.

\section{Hyperparameters of MC-DDQN} \label{apdx:hyperparams}
All hyperparameters are listed in \autoref{table:hyperparams}. Our algorithm introduces several additional hyperparameters beyond those typically used by standard RL algorithms:

\begin{table}[H]
\centering
\caption{Hyperparameters used for the experiments.}
\label{table:hyperparams}
\rowcolors{2}{white}{gray!15}
\begin{tabular}{ p{7cm} p{1.5cm}  }
\toprule
\textbf{Hyperparameter Name}           & \textbf{Value} \\
$Q_\theta$ and $R_\psi$ hidden layers   & 2              \\
$Q_\theta$ and $R_\psi$ hidden layer size  & 128            \\
$Q_\theta$ and $R_\psi$ activation function  & ReLU            \\
$Q_\theta$ and $R_\psi$ optimizer  & Adam            \\
$Q_\theta$ learning rate  & 0.0001            \\
$R_\psi$ learning rate  & 0.01            \\
$Q_\theta$ loss                         & SmoothL1        \\
$R_\psi$ loss                      & $L_2$            \\
Batch size                             & 32             \\
Discount factor $\gamma$               & 0.95           \\
Training steps on \emph{Safe}        & 10000          \\
Training steps on \emph{Full}        & 10000          \\
Replay buffer size                     & 10000          \\
Exploration steps                      & 1000           \\
Exploration $\epsilon_\emph{start}$  & 1.0            \\
Exploration $\epsilon_\emph{end}$    & 0.05           \\
Target network EMA coefficient             & 0.005          \\
Forecasting training steps $l$     & 5000           \\
Scoring rollout steps $n$      & 30             \\
Number of scoring rollouts $k$ & 20             \\
Predicted reward difference threshold $\delta_r$      & 0.05             \\
Add transitions from transition model  & False           \\
\bottomrule
\end{tabular}
\vspace{2mm}
\end{table}

\paragraph{Reward model architecture and learning rate.} Hyperparameters specify the architecture and learning rate of the reward model $R_\psi$. Since learning a reward model is a supervised learning task, these hyperparameters can be tuned on a dataset of transitions collected by any policy. The reward model architecture may be chosen to match the Q-function $Q_\theta$.

\paragraph{Forecasting training steps $l$.} This parameter describes the number of updates to the Q-function needed to predict the future policy based on a new transition. As shown in \autoref{fig:tam_train_steps}, this value must be sufficiently large to update the learned values and corresponding policy. It can be selected by artificially adding a transition that alters the optimal policy and observing the number of training steps required to learn the new policy.

\paragraph{Scoring rollout steps $n$.} This parameter controls the length of the trajectories used to compare two forecasted policies. Longer trajectories provide more direct behavioral information about each policy and reduce reliance on the value-function bootstrap. The error decomposition in \Cref{apdx:proofs} confirms that, under expected conditions ($\epsilon_Q > \frac{\epsilon_R}{1-\gamma}$), the evaluator bound is monotonically decreasing in~$n$, favoring longer rollouts. In episodic tasks with access to a simulator, a safe choice is the maximum episode length; in continuing tasks, a truncation horizon typically used in training may be suitable. When using a learned transition model, compounding prediction error provides a practical upper bound on useful~$n$. Computational costs scale linearly in~$n$ and can be reduced by choosing a smaller value based on domain knowledge.

\paragraph{Number of scoring rollouts $k$.} This parameter specifies the number of trajectories obtained by rolling out each forecasted policy for comparison. The required number depends on the stochasticity of the environment and policies. If both the policy and environment are deterministic, $k$ can be set to 1. Otherwise, $k$ can be selected using domain knowledge or by measuring how much rollouts are required to distinguish between multiple known policies with different behaviors.

\paragraph{Predicted reward difference threshold $\delta_r$.} This threshold defines the minimum difference between the predicted and observed rewards for a transition to trigger a check. As discussed in \Cref{sec:sensitivity}, this parameter is introduced only to reduce computations and can be set to 0. However, it can be adjusted based on domain knowledge to speed up training by minimizing unnecessary checks. The key requirement is that any reward hacking behavior must increase the reward by more than this threshold relative to the reward predicted by the reward model. In our gridworld and Reacher experiments, 0.05 performed well when rewards were normalized to $[-1, 1]$. For Ant and HalfCheetah, where the reward scale is larger, we use $\delta_r{=}2.0$.

\subsection{Environment-specific Parameters}

The training steps in \emph{Box Moving} were reduced to speed up training. \emph{Tomato Watering} has many stochastic transitions because each tomato has a chance of drying out at each step. To increase the robustness of evaluations, we increased the number of scoring rollouts $k$. \emph{Rocks and Diamonds} required more steps to converge to the optimal policy. Additionally, using the transition model to collect fresh data while forecasting in \emph{Rocks and Diamonds} makes reward hacking detection more reliable. Each gridworld environment's rewards were scaled to $[-1, 1]$.
\begin{table}[H]
\centering
\caption{Environment-specific hyperparameter overrides.}
\label{table:env-hyperparams}
\rowcolors{2}{white}{gray!15}
\begin{tabular}{ p{7cm} p{1.5cm}  }
\toprule
\textbf{Hyperparameter Name}       & \textbf{Value} \\
\midrule
\rowcolor{white}
\multicolumn{2}{c}{Box Moving} \\
\midrule
Training steps on \emph{Safe}    & 1000           \\
Training steps on \emph{Full}    & 1000           \\
Replay buffer size                 & 1000           \\
Exploration steps                  & 100            \\
Forecasting training steps $l$ & 500            \\
\rowcolor{white}
\midrule
\multicolumn{2}{c}{Absent Supervisor} \\
\midrule
Number of supervisors                  & 1          \\
Remove walls                           & False      \\
\rowcolor{white}
\midrule
\multicolumn{2}{c}{Tomato Watering} \\
\midrule
Number of scoring rollouts $k$ & 100        \\
\rowcolor{white}
\midrule
\multicolumn{2}{c}{Rocks and Diamonds} \\
\midrule
Training steps on \emph{Safe}        & 15000      \\
Training steps on \emph{Full}        & 15000      \\
Forecasting training steps $l$     & 7500      \\
Add transitions from transition model  & True      \\
\bottomrule
\end{tabular}
\vspace{2mm}
\end{table}

\subsection{Hyperparameters of MC-TD3}

\begin{table}[H]
\centering
\caption{Hyperparameters used for the MC-TD3 experiments (Reacher, Ant, HalfCheetah).}
\label{table:td3-hyperparams}
\rowcolors{2}{white}{gray!15}
\begin{tabular}{ p{7cm} p{1.5cm}  }
\toprule
\textbf{Hyperparameter Name}           & \textbf{Value} \\
Actor, critic, and reward model hidden layers   & 2              \\
Actor, critic, and reward model hidden layer size  & 256            \\
Actor, critic, and reward model activation function  & ReLU            \\
Actor, critic, and reward model optimizer  & Adam            \\
Actor and critic learning rate  & 0.0003            \\
$R_\psi$ learning rate  & 0.003            \\
Batch size                             & 256             \\
Discount factor $\gamma$               & 0.99           \\
Training steps       & 200000          \\
Replay buffer size                     & 200000          \\
Exploration steps                      & 30000           \\
Target networks EMA coefficient             & 0.005          \\
Policy noise  & 0.01 \\
Exploration noise & 0.1 \\
Policy update frequency & 2 \\
Forecasting training steps $l$     & 10000           \\
Scoring rollout steps $n$      & 50             \\
Number of scoring rollouts $k$ & 100             \\
Predicted reward difference threshold $\delta_r$      & 0.05             \\
\bottomrule
\end{tabular}
\vspace{2mm}
\end{table}

We did not perform extensive hyperparameter tuning; most hyperparameters are inherited from the implementation provided by~\citet{huang2022cleanrl}. The same MC-TD3 hyperparameters are used for all three continuous-control environments (Reacher, Ant, HalfCheetah), with the following environment-specific overrides:

\begin{table}[H]
\centering
\caption{Environment-specific hyperparameter overrides for MC-TD3.}
\label{table:td3-env-hyperparams}
\rowcolors{2}{white}{gray!15}
\begin{tabular}{ p{7cm} p{1.5cm}  }
\toprule
\textbf{Hyperparameter Name}       & \textbf{Value} \\
\midrule
\rowcolor{white}
\multicolumn{2}{c}{Ant \& HalfCheetah} \\
\midrule
Training steps       & 1000000          \\
Replay buffer size                     & 1000000          \\
Exploration steps                      & 100000           \\
Scoring rollout steps $n$      & 1000             \\
Predicted reward difference threshold $\delta_r$ & 2.0 \\
\bottomrule
\end{tabular}
\vspace{2mm}
\end{table}
Ant and HalfCheetah require more training steps (1M vs.\ 200k) because TD3 does not converge within 200k steps in these environments. The exploration budget is scaled proportionally. The reward difference threshold $\delta_r$ is increased from 0.05 to 2.0 to match the larger reward scale. The scoring rollout steps $n$ is set to the full episode length (1000).

\end{document}